\documentclass{colt2017} % Include author names

\usepackage{hyperref}

\usepackage{dsfont}
\usepackage{xparse}

\newcommand{\lr}[1]{\left (#1\right)}
\newcommand{\lrc}[1]{\left \{#1\right\}}
\newcommand{\lrs}[1]{\left [#1 \right]}

\renewcommand{\P}[1]{\mathbb P\lrc{#1}}
\NewDocumentCommand{\E}{o}{\mathbb E\IfValueT{#1}{\lrs{#1}}}
\NewDocumentCommand{\1}{o}{\mathds 1\IfValueT{#1}{\!\!\lrc{#1}}}

\usepackage{xspace}

\usepackage{color}

\usepackage{algorithm}
\usepackage{algorithmic}

\newcommand{\F}{\mathcal{F}}
\newcommand{\event}{\mathcal{E}}

\newcommand{\tmin}{t_{\texttt{min}}}

\DeclareMathOperator{\UCB}{UCB}
\DeclareMathOperator{\LCB}{LCB}
\newcommand{\DLCB}{\hat \Delta^{\texttt{LCB}}}

\newcommand{\EXPPP}{EXP3++\xspace}
\newcommand{\EXP}{EXP3\xspace}

\renewcommand{\cite}{\citep}

\title[An Improved Parametrization and Analysis of the EXP3++ Algorithm]{An Improved Parametrization and Analysis of the EXP3++ Algorithm for Stochastic and Adversarial Bandits}

\usepackage{times}
 \coltauthor{\Name{Yevgeny Seldin} \Email{seldin@di.ku.dk}\\
 \addr University of Copenhagen
 \AND
 \Name{G{\'a}bor Lugosi} \Email{gabor.lugosi@gmail.com}\\
 \addr ICREA and Pompeu Fabra University
 }

\begin{document}

\maketitle

\begin{abstract}
We present a new strategy for gap estimation in randomized algorithms for multiarmed bandits and combine it with the \EXPPP algorithm of \citet{SS14}. In the stochastic regime the strategy reduces dependence of regret on a time horizon from $(\ln t)^3$ to $(\ln t)^2$ and eliminates an additive factor of order $\Delta e^{1/\Delta^2}$, where $\Delta$ is the minimal gap of a problem instance. In the adversarial regime regret guarantee remains unchanged.
\end{abstract}

\section{Introduction}

Stochastic (i.i.d.) and adversarial multiarmed bandits are two of the most basic problems in online learning \cite{Tho33,Rob52,LR85,ACBF02,ACB+95,ACB+02}. In recent years there has been an increased interest in algorithms that can be applied in both settings \cite{BS12,SS14,AC16}. This line of work can be seen as part of a growing area of ``general-purpose'' algorithms that are applicable to multiple online learning settings simultaneously \cite{Sel15}. The advantage of such algorithms is in their ability to exploit problem simplicity (such as i.i.d.\ environment) without compromising on the worst case guarantees.

There exist two basic approaches to deriving algorithms applicable to both stochastic and adversarial multiarmed bandits. The first starts with an algorithm for stochastic bandits and equips it with a mechanism for detecting deviations from the i.i.d.\ assumption. If such a deviation is detected, the algorithm switches into an adversarial operation mode \citep{BS12,AC16}. The switch is irreversible and, therefore, this approach relies on a knowledge of time horizon. It allows to achieve $O\lr{\sum_{a:\Delta(a)>0} \frac{\ln T}{\Delta(a)}}$ regret guarantee in the stochastic regime and $O\lr{\sqrt{KT\ln T}}$ regret guarantee in the adversarial regime, where $a$ indexes the arms, $\Delta(a)$ is the suboptimality gap of arm $a$, $T$ is the number of game rounds, and $K$ is the number of arms \citep{AC16}. We note that in absence of the knowledge of time horizon the approach has to be combined with the doubling trick, which leads to deterioration of regret guarantee in the stochastic regime to $O\lr{\sum_{a:\Delta(a)>0} \frac{\lr{\ln t}^2}{\Delta(a)}}$ (we use capital $T$ in results that assume a known time horizon and small $t$ otherwise).

The second approach is to start with an algorithm for adversarial bandits and modify its exploration strategy to allow for gap detection. This approach has a number of advantages and disadvantages. On the positive side it has a single operation mode that naturally takes care of both regimes; it does not rely on the knowledge of  time horizon; it has a better regret guarantee of $O\lr{\sqrt{Kt\ln K}}$ in the adversarial regime; and it can handle additional intermediate regimes, such as moderately contaminated stochastic regime and adversarial regime with a gap \citep{SS14}. On the negative side its current regret guarantee in the stochastic regime is weaker, $O\lr{\sum_{a:\Delta(a) > 0} \frac{\lr{\ln t}^3}{\Delta(a)}}$ with an exponentially large additive constant, and it does not provide high-probability regret guarantee in the adversarial regime, but only a guarantee on the expected regret. In our contribution we modify the second approach and improve its regret guarantee in the stochastic regime by a multiplicative factor of $\ln t$, as well as eliminate the exponentially large additive constant.

The work of \citet{SS14} is based on an observation that the \EXP algorithm with losses for adversarial multiarmed bandits \citep{ACB+02,BCB12} has a degree of freedom in the choice of exploration strategy. \citeauthor{SS14} have proposed a generalized \EXPPP algorithm based on a combination of two independent mechanisms. The first mechanism controls the performance of the algorithm in adversarial environments through a standard \EXP-like playing strategy in the form of a Gibbs distribution over actions. The second mechanism exploits the residual degree of exploration freedom for detection and exploitation of suboptimality gaps. The two mechanisms operate in parallel with almost no interference and achieve improved regret guarantee in the stochastic regime without impairing the adversarial regret guarantee. 

We propose a new generic strategy for gap estimation in the stochastic regime that can be combined with almost any randomized playing strategy, including the \EXPPP algorithm. The new strategy is based on unweighted losses, as opposed to importance-weighted losses used in the main result of \citet[Theorem 3]{SS14}. It improves over the attempt of \citeauthor{SS14} to use unweighted losses for gap estimation \citep[Theorem 4]{SS14}, both in terms of regret bound and in terms of underlying assumptions (the regret bound is improved by a multiplicative factor of order $\frac{\ln t}{\Delta}$ and the assumption on known time horizon is eliminated). 

The proposed approach is modular: we provide an algorithm for gap estimation in the i.i.d.\ regime and then combine it with the \EXPPP algorithm, which provides protection against an adversary. 
%We show that in the i.i.d.\ regime the exploration strategy of the proposed gap estimation algorithm contributes $O\lr{\sum_{a:\Delta(a)>0} \frac{\ln t}{\Delta(a)}}$ to the regret. In combination with the \EXPPP algorithm the regret of the combined algorithm in the i.i.d.\ regime is $O\lr{\sum_{a:\Delta(a)>0} \frac{\lr{\ln t}^2}{\Delta(a)}}$, which constitutes an improvement by a factor of $\ln t$ over the result of \citet{SS14}. The additive constant is also reduced from exponential in $1/\Delta^2$ to polynomial in $1/\Delta$. 
The key features of the contribution are summarized below:
%
%Let $K$ be the number of arms, let $t$ be the number of game rounds, let $a$ index the arms, let $\Delta(a)$ denote the gap of arm $a$ and $\Delta$ denote the minimal gap in the stochastic regime. The \EXPPP algorithm presented in \citet{SS14} has $O\lr{\sqrt{Kt\ln K}}$ regret in the adversarial regime and $O\lr{\sum_{a:\Delta(a)>0} \frac{\lr{\ln t}^3}{\Delta(a)}}$ regret in the stochastic regime with an additive factor of order $e^{1/\Delta^2}$. We present a new parametrization and analysis of the \EXPPP algorithm that improves the regret in the stochastic regime to $O\lr{\sum_{a:\Delta(a)>0} \frac{\lr{\ln t}^2}{\Delta(a)}}$ and reduces the additive factor from exponential to polynomial while leaving the adversarial guarantee unchanged. The result has three positive and one negative aspect:
\begin{itemize}
	\item[$+$] We propose a novel generic strategy for gap estimation by randomized algorithms in i.i.d.\ regimes. Our %approach is modular and the 
	strategy can be combined with any randomized algorithm that has the necessary freedom in the choice of exploration distribution. 
%	\item[$+$] In the i.i.d.\ regime the gap estimation strategy contributes $O\lr{\sum_{a:\Delta(a)>0} \frac{\ln t}{\Delta(a)}}$ to the regret. (The total regret depends on the algorithm it is combined with.) In the adversarial regime the contribution is $O\lr{\sqrt{Kt\ln K}}$.
	\item[$+$] In combination with the \EXPPP the regret of the combined algorithm in the i.i.d.\ regime is of order $O\lr{\sum_{a:\Delta(a)>0} \frac{\lr{\ln t}^2}{\Delta(a)}}$, which is an improvement by a multiplicative factor of $\ln t$ compared to \citet{SS14}. In the adversarial regime the regret guarantee is unchanged, $O\lr{\sqrt{Kt\ln K}}$.
%	\item[$+$] In comparison to SOA algorithm suggested in \citet{BS12} the regret of the new parametrization of \EXPPP is strictly better in the adversarial regime and identical in the stochastic regime.
	\item[$-$] The new approach does \emph{not} provide an improved regret guarantee in the moderately contaminated stochastic regime and adversarial regime with a gap defined in \citet{SS14}. For both regimes only the worst-case adversarial regret guarantee holds.
	\item Without the assumption on known time horizon the regret guarantee in the stochastic regime is of the same order as \citet{AC16} and the regret guarantee in the adversarial regime is stronger by a factor of $\sqrt{\ln t}$. However, our approach does not provide high-probability guarantee in the adversarial regime, as do \citet{AC16}. But on the positive side our approach is modular, it does not depend on the time horizon, and has a single operation mode for both stochastic and adversarial regimes, which makes it a bit more elegant.
%	\item The analysis of the proposed strategy is based on an interplay between high-probability and in-expectation analysis. This technique may potentially turn useful beyond the scope of our work.
\end{itemize}

In the following we start with outlining the problem setting in Section~\ref{sec:setting} and cite the \EXPPP algorithm and known results about it in Section~\ref{sec:known}. We present our gap estimation strategy in Section~\ref{sec:gap} and its combination with \EXPPP in Section~\ref{sec:EXP3}. The corresponding proofs are given in Sections~\ref{sec:thmDLCB} and~\ref{sec:thmEXP3++} and we finish with a discussion in Section~\ref{sec:discussion}.

\section{Problem Setting}
\label{sec:setting}

The problem setting follows \citet{SS14}. We study the multiarmed bandit game. At round $t$ of the game the algorithm chooses an \emph{action} $A_t$ among $K$ possible actions (a.k.a.\ \emph{arms}) and observes the corresponding loss $\ell_t^{A_t}$. The losses of other arms are not observed. There is a large number of loss generation models two of which, stochastic and adversarial, are considered below. In this work we restrict ourselves to loss sequences $\lrc{\ell_t^a}_{t,a}$ that are generated independently of the algorithm's actions (the so called oblivious learning model). Under this assumption we can assume that the loss sequences are determined before the game starts (but not revealed to the algorithm). We also make the standard assumption that the losses are bounded in the $[0,1]$ interval.

The performance of an algorithm is quantified by the \emph{expected regret}, defined as the difference between the expected loss of the algorithm up to round $t$ and the expected loss of the best arm up to round $t$:
\begin{equation}
\label{eq:regret}
R(t) = \sum_{s=1}^t \E[\ell_s^{A_s}] - \min_a \lrc{\E[\sum_{s=1}^t \ell_s^a]}.
\end{equation}
The expectation is taken over the possible randomness of the algorithm and the loss generation model. In the i.i.d.\ setting the $\ell_s^a$-s are random variables and the definition coincides with the definition of pseudo regret \citep{BCB12}. In the adversarial setting the $\ell_s^a$-s are considered deterministic and the second expectation can be omitted. In some literature $R(t)$ is termed \emph{excess of cumulative predictive risk} \citep{Win17}. Since $R(t)$ is the only notion of regret considered in the paper we will often call it simply \emph{regret} (omitting the word ``expected''). The goal of the algorithm is to minimize $R(t)$.

We consider two standard loss generation models, the \emph{adversarial regime} and the \emph{stochastic regime}.

\paragraph{Adversarial regime} In this regime the loss sequences are generated by an unrestricted adversary (who is oblivious to the algorithm's actions). %This is the most general setting and the other three regimes can be seen as special cases of the adversarial regime. 
An arm $a\in\arg \min_{a'} \lr{\sum_{s=1}^t \ell_s^{a'}}$ is known as a \emph{best arm in hindsight} for the first $t$ rounds.

\paragraph{Stochastic regime} In this regime the losses $\ell_t^a$ are sampled independently from an unknown distribution that depends on $a$, but not on $t$. We use $\mu(a) = \E[\ell_t^a]$ to denote the expected loss of an arm $a$. An arm $a$ is called a \emph{best arm} if $\mu(a) = \min_{a'}\lrc{\mu(a')}$ and \emph{suboptimal} otherwise; let $a^*$ denote some best arm. For each arm $a$, define the \emph{gap} $\Delta(a) = \mu(a) - \mu(a^*)$. %Let $\Delta = \min_{a: \Delta(a) > 0} \lrc{\Delta(a)}$ denote the minimal gap.

Letting $N_t(a)$ be the number of times arm $a$ was played up to (and including) round $t$, in the stochastic regime the regret can be rewritten as
\[
R(t) = \sum_a \E[N_t(a)] \Delta(a).
\]

%We add a couple of additional definitions that will be used in our modification. Let $\hat L_t(a) = \sum_{s=1}^t \1[A_t=a] \ell_t^a$ denote the cumulative loss of arm $a$ on the rounds that it was played and $\hat \mu_t(a) = \frac{\hat L_t(a)}{N_t(a)}$ the empirical estimate of $\mu(a)$ on round $t$. If $N_t(a) = 0$ we define $\hat \mu_t(a) = 1$.

%Remark about the notation: we use ``hat'' to denote quantities based on ``unweighted'' statistics, such as $\hat \mu_t(a)$ and $\hat L_t(a)$ and ``tilde'' to denote quantities based on importance-weighted sampling, such as $\tilde L_t(a)$.

\section{Known Results}
\label{sec:known}

In our work we are using the \EXPPP algorithm of \citet{SS14}, which is provided in Algorithm \ref{algo:EXP3++}. 

\begin{algorithm}
\begin{algorithmic}[1]
\STATE \textit{Remark: See text for definition of $\eta_t$ and $\xi_t(a)$; $\1[\cdot]$ is used to denote the indicator function}
\STATE $\forall a$: $\tilde L_0(a) = 0$\\
\FOR{$t = 1,2,...$}
\STATE $\forall a$: $\varepsilon_t(a) = \min \lrc{\frac{1}{2K}, \frac{1}{2}\sqrt{\frac{\ln K}{t K}}, \xi_t(a)}$\\
\STATE $\forall a$: $\rho_t(a) = e^{-\eta_t \tilde L_{t-1}(a)} \big / \sum_{a'} e^{-\eta_t \tilde L_{t-1}(a')}$\label{line:EXP3rho}\\
\STATE $\forall a$: $\tilde \rho_t(a) = \lr{1 - \sum_{a'} \varepsilon_t(a')} \rho_t(a) + \varepsilon_t(a)$\label{line:EXP3rho+1}\\
\STATE Draw action $A_t$ according to $\tilde \rho_t$ and play it\\
\STATE Observe and suffer the loss $\ell_t^{A_t}$\label{line:EXP3rho+3}\\
\STATE $\forall a: ~ \tilde \ell_t^a = \frac{\ell_t^{A_t}}{\tilde \rho_t(a)}\1[A_t = a]$\\
\STATE $\forall a: ~ \tilde L_t(a) = \tilde L_{t-1}(a) + \tilde \ell_t^a$
\ENDFOR
\end{algorithmic}
\caption{\EXPPP.}
\label{algo:EXP3++}
\end{algorithm}

%\begin{algorithm}[H]
%%\begin{algorithmic}[1]
%\emph{Remark: See text for definition of $\eta_t$ and $\xi_t(a)$}\\
%$\forall a$: $\tilde L_0(a) = 0$\\
%\For{$t = 1,2,...$}{
%$\forall a$: $\varepsilon_t(a) = \min \lrc{\frac{1}{2K}, \frac{1}{2}\sqrt{\frac{\ln K}{t K}}, \xi_t(a)}$\\
%$\forall a$: $\rho_t(a) = e^{-\eta_t \tilde L_{t-1}(a)} \big / \sum_{a'} e^{-\eta_t \tilde L_{t-1}(a')}$\\
%$\forall a$: $\tilde \rho_t(a) = \lr{1 - \sum_{a'} \varepsilon_t(a')} \rho_t(a) + \varepsilon_t(a)$\\
%Draw action $A_t$ according to $\tilde \rho_t$ and play it\\
%Observe and suffer the loss $\ell_t^{A_t}$\\
%$\forall a: ~ \tilde \ell_t^a = \frac{\ell_t^{A_t}}{\tilde \rho_t(a)}\1_t^a$\\
%$\forall a: ~ \tilde L_t(a) = \tilde L_{t-1}(a) + \tilde \ell_t^a$
%}
%%\end{algorithmic}
%\caption{\EXPPP.}
%\label{algo:EXP3++}
%\end{algorithm}

We note that since we are not changing the \EXPPP algorithm, but only modify the definition of the exploration parameters $\xi_t(a)$, the following result of \citet{SS14} is valid.
\begin{theorem}[\citealp{SS14}]
\label{thm:EXP3++Adv}
For $\eta_t = \frac{1}{2}\sqrt{\frac{\ln K}{t K}}$ and any $\xi_t(a)\geq 0$ the regret of the \EXPPP in the adversarial regime for any $t$ satisfies:
\[
R(t) \leq 4 \sqrt{K t \ln K}.
\]
\end{theorem}

Note that the regret bound in Theorem~\ref{thm:EXP3++Adv} is just a factor of 2 worse than the regret of \EXP with losses \citep{BCB12}.

\section{Gap Estimation in Randomized Playing Strategies}
\label{sec:gap}

%We define operators $\min^*(x,y)$ and $\max^*(x,y)$ in the following way: if $y$ is defined then $\min^*$ is identical to $\min$ and $\max^*$ is identical to $\max$, otherwise $\min^*$ and $\max^*$ take the value of $x$.
Our first contribution is a generic algorithm for gap estimation in stochastic environments. The algorithm can be combined with any randomized playing strategy, including the \EXPPP. It is detailed in Algorithm~\ref{algo:IWUCB}. Line~\ref{line:anyrho} is the ``plug-in'' point, where the algorithm can be combined with any randomized playing strategy. In combination with the \EXPPP we replace Line~\ref{line:anyrho} in Algorithm~\ref{algo:IWUCB} with Line~\ref{line:EXP3rho} from Algorithm~\ref{algo:EXP3++}. (Note that lines~\ref{line:anyrho+1}-\ref{line:anyrho+3} in Algorithm~\ref{algo:IWUCB} are identical to lines~\ref{line:EXP3rho+1}-\ref{line:EXP3rho+3} in Algorithm~\ref{algo:EXP3++} and thus the two mechanisms can operate in parallel without interfering with each other.)

We use $\hat L_t(a)$ to denote unweighted cumulative loss of arm $a$ up to (and including) round $t$. (It should not be confused with $\tilde L_t(a)$, which denotes cumulative importance-weighted loss and defined in Lines 9-10 of Algorithm~\ref{algo:EXP3++}.)

%%\begin{algorithm}[t]
%%Play each arm once and update $\hat L_K(a)$ and $N_K(a)$
%%%\STATE $\forall a$: $L_1(a) = 0$.
%%%\STATE $\forall a$: $N_1(a) = 0$.
%%\For{$t = K+1,K+2,...$}{
%%$\forall a$: $\UCB_t(a) = \min\bigg\{1, \frac{\hat L_{t-1}(a)}{N_{t-1}(a)} + \sqrt{\frac{\alpha \ln t}{2N_{t-1}(a)}}\bigg\}$
%%$\forall a$: $\LCB_t(a) = \max\bigg\{0, \frac{\hat L_{t-1}(a)}{N_{t-1}(a)} - \sqrt{\frac{\alpha \ln t}{2N_{t-1}(a)}}\bigg\}$
%%$\forall a$:\\ $\DLCB_t(a) = \max\lrc{0,\LCB_t(a) - {\displaystyle\min_{a'}}\UCB_t(a')}$\label{ln:DLCB}
%%$\forall a$: $\varepsilon_t(a) = \min \lrc{\frac{1}{2K}, \frac{1}{2}\sqrt{\frac{\ln K}{tK}}, \xi_t(a)}$
%%Let $\rho_t(a)$ be any distribution over $\lrc{1,\dots,K}$.\label{line:anyrho} ~~//\textit{This is the plug-in point for other algorithms}
%%$\forall a$: $\tilde \rho_t(a) = \lr{1 - \sum_{a'} \varepsilon_t(a')} \rho_t(a) + \varepsilon_t(a)$
%%Draw action $A_t$ according to $\tilde \rho_t$ and play it
%%Observe and suffer the loss $\ell_t^{A_t}$
%%$\forall a: \hat L_t(a) = \hat L_{t-1}(a) + \ell_t^a \1[A_t=a]$
%%$\forall a: N_t(a) = N_{t-1}(a) + \1[A_t=a]$
%%}
%%\caption{Gap Estimation in Randomized Playing Strategies.}
%%\label{algo:IWUCB}
%%\end{algorithm}

\begin{algorithm}
\begin{algorithmic}[1]
\STATE \textit{Remark: see text for definition of $\xi_t(a)$}
\STATE Play each arm once and update $\hat L_K(a)$ and $N_K(a)$
%\STATE $\forall a$: $L_1(a) = 0$.
%\STATE $\forall a$: $N_1(a) = 0$.
\FOR{$t = K+1,K+2,...$}
\STATE $\forall a$: $\UCB_t(a) = \min\bigg\{1, \frac{\hat L_{t-1}(a)}{N_{t-1}(a)} + \sqrt{\frac{\alpha \ln \lr{t K^{1/\alpha}}}{2N_{t-1}(a)}}\bigg\}$
\STATE $\forall a$: $\LCB_t(a) = \max\bigg\{0, \frac{\hat L_{t-1}(a)}{N_{t-1}(a)} - \sqrt{\frac{\alpha \ln \lr{t K^{1/\alpha}}}{2N_{t-1}(a)}}\bigg\}$
\STATE $\forall a$: $\DLCB_t(a) = \max\lrc{0,\LCB_t(a) - {\displaystyle\min_{a'}}\UCB_t(a')}$\label{ln:DLCB}\hfill//\textit{ Note that $0 \leq \DLCB_t(a) \leq 1$}
\STATE $\forall a$: $\varepsilon_t(a) = \min \lrc{\frac{1}{2K}, \frac{1}{2}\sqrt{\frac{\ln K}{tK}}, \xi_t(a)}$
\STATE Let $\rho_t(a)$ be any distribution over $\lrc{1,\dots,K}$\label{line:anyrho}\hfill//\textit{~The plug-in point for other algorithms}
\STATE $\forall a$: $\tilde \rho_t(a) = \lr{1 - \sum_{a'} \varepsilon_t(a')} \rho_t(a) + \varepsilon_t(a)$\label{line:anyrho+1}
\STATE Draw action $A_t$ according to $\tilde \rho_t$ and play it
\STATE Observe and suffer the loss $\ell_t^{A_t}$\label{line:anyrho+3}
\STATE $\forall a: \hat L_t(a) = \hat L_{t-1}(a) + \ell_t^a \1[A_t=a]$
\STATE $\forall a: N_t(a) = N_{t-1}(a) + \1[A_t=a]$
\ENDFOR
\end{algorithmic}
\caption{Gap Estimation in Randomized Playing Strategies.}
\label{algo:IWUCB}
\end{algorithm}

We provide the following guarantee for empirical gap estimates $\DLCB_t(a)$ in Line~\ref{ln:DLCB} of Algorithm~\ref{algo:IWUCB}.

\begin{proposition}
\label{thm:DLCB}
For any $a$ and $t$, the gap estimates $\DLCB_t(a)$ of Algorithm~\ref{algo:IWUCB} in the i.i.d.\ regime satisfy:
\[
\P{\DLCB_t(a) \geq \Delta(a)} \leq \frac{1}{t^{\alpha - 1}}.
\]
Furthermore, for any choice of $\xi_t(a)$, such that $\xi_t(a) \geq \frac{\beta\ln t}{t \DLCB_t(a)^2}$, for $\alpha \geq 3$, for $\beta \geq 64(\alpha+1) \geq 256$, and $t \geq \tmin(a) := \min\lrc{t: t\geq \frac{4K\beta \lr{\ln t}^2}{\Delta(a)^4 \ln K}}$ (this is the first time when $\frac{\beta \ln t}{t \Delta(a)^2} \leq \frac{1}{2} \sqrt{\frac{\ln K}{tK}}$) the gap estimates satisfy
\[
\P{\DLCB_t(a) \leq \frac{1}{2} \Delta(a)} \leq \lr{\frac{\ln t}{t \Delta(a)^2}}^{\alpha-2} + \frac{2}{Kt^{\alpha-1}} + 2\lr{\frac{1}{t}}^{\frac{\beta}{8}}.
\]
\end{proposition}
A proof of this proposition is provided in Section~\ref{sec:thmDLCB}. The main message of Proposition~\ref{thm:DLCB} is that for an appropriate choice of $\xi_t(a)$ the gap estimates $\DLCB_t(a)$ satisfy $\frac{1}{2} \Delta(a) \leq \DLCB_t(a) \leq \Delta(a)$ with high probability. Thus, $\DLCB_t(a)$ can be used as a reliable estimate of $\Delta(a)$ for any higher level purpose. 

%The cost of exploration done by Algorithm~\ref{algo:IWUCB}, i.e.\ the regret due to playing suboptimal arms suggested by Algorithm~\ref{algo:IWUCB}, is $\sum_{a:\Delta(a)>0} \Delta(a) \sum_{s=1}^t \varepsilon_t(a)$.  In the i.i.d.\ regime it is governed by $\displaystyle \sum_{a:\Delta(a)>0} \Delta(a) \sum_{s=1}^t \xi_s(a)$ and contributes $O\lr{\sum_{a:\Delta(a) > 0} \frac{\ln t}{\Delta(a)}}$ to the regret.

\section{Reparametrization and Improved Regret Guarantee for \EXPPP}
\label{sec:EXP3}

We combine Algorithm~\ref{algo:IWUCB} with the \EXPPP algorithm to achieve an improved regret guarantee in the stochastic regime.

\begin{theorem}
\label{thm:EXP3++}
Let $\xi_t(a) = \frac{\beta \ln t}{t \DLCB_t(a)^2}$, where $\DLCB_t(a)$ is the empirical gap estimate from Algorithm~\ref{algo:IWUCB}. Then for $\alpha = 3$ and $\beta = 256$ the expected regret of \EXPPP in the stochastic regime satisfies
\[
R_t = O\lr{\sum_{a:\Delta(a)>0} \frac{\lr{\ln t}^2}{\Delta(a)}} + \tilde O\lr{\sum_{a:\Delta(a)>0} \frac{K}{\Delta(a)^3}},
\]
where the $\tilde O$ notation hides logarithmic factors.
\end{theorem}

A proof of this theorem is provided in Section~\ref{sec:thmEXP3++}. We note that the regret guarantee of \EXPPP in the adversarial regime scales with $\sqrt t$. Therefore, the ``logarithmic'' regret guarantee in the stochastic regime becomes interesting when $\frac{(\ln t)^2}{\Delta(a)} < \sqrt{t}$ or $t \geq \tilde \Omega\lr{\frac{1}{\Delta(a)^2}}$ (where the tilde notation hides logarithmic factors). The second term in the regret bound in Theorem~\ref{thm:EXP3++} comes from the initial period of the game, where reliable estimate of the gaps cannot be achieved. The value of this term is only slightly suboptimal.

Theorem~\ref{thm:EXP3++} improves the regret bound of \citet[Theorem 3]{SS14} by a multiplicative factor of $\ln t$ and eliminates an exponentially large additive constant of order $\Delta e^{1/\Delta^2}$. We note that asymptotically the regret bound in Theorem~\ref{thm:EXP3++} matches the oracle bound in \citet[Theorem 2]{SS14}, where knowledge of the gaps $\Delta(a)$ is assumed. %Thus, even though the regret of the gap estimation strategy in Algorithm~\ref{algo:IWUCB} scales as $O\lr{\sum_{a:\Delta(a) > 0} \frac{\ln t}{\Delta(a)}}$, it does not seem possible to avoid the squared logarithmic term $\lr{\ln t}^2$ when it is combined with the \EXPPP. We get back to this point in the discussion of the paper.

\section{Proof of Proposition~\ref{thm:DLCB}}
\label{sec:thmDLCB}

The proof is based on four steps. In the first step we show that $\DLCB_t(a) \leq \Delta(a)$ with high probability. In the second step we derive a high-probability lower bound on the exploration parameters $\varepsilon_t(a)$. In the third step we derive a high-probability lower bound on the number of times $N_t(a)$ each arm is played. Finally, in the last step we show that $\DLCB_t(a) \geq \frac{1}{2} \Delta(a)$ with high probability.

%\begin{proof}
\paragraph{Step 1: An upper bound for $\DLCB_t(a)$.} The following property of upper and lower confidence bounds follows by standard arguments, as in \citet{ACBF02}. (The proof is standard and provided in the appendix for completeness.)
\begin{lemma}
\label{cl:1}
For any $a$ and $t \geq K$:
\begin{align*}
\P{\UCB_t(a) \leq \mu(a)} &\leq \frac{1}{Kt^{\alpha-1}},\\
\P{\LCB_t(a) \geq \mu(a)} &\leq \frac{1}{Kt^{\alpha-1}}.\\
\end{align*}
\end{lemma}

\begin{corollary}
For any $a$ and $t \geq K$:
\[
\P{\DLCB_t(a) \geq \Delta(a)} \leq \frac{1}{t^{\alpha - 1}}.
\]
\label{cor:DLCB}
\end{corollary}

\begin{proof}
\[
\P{\DLCB_t(a) \geq \Delta(a)} \leq \P{\LCB_t(a) \geq \mu(a)} + \sum_{a'\neq a} \P{\UCB_t(a') \leq \mu(a')} \leq \frac{1}{t^{\alpha - 1}}.
\]
\end{proof}

%\begin{corollary}
%\label{cor:1}
%Let $F(a)$ be the number of rounds when $\DLCB_t(a) > \Delta(a)$. Then for $\alpha \geq 3$ we have $\E[F(a)] \leq \frac{\pi^2}{3}$.
%\end{corollary}

\paragraph{Step 2: A lower bound for $\varepsilon_t(a)$.} We have $\tilde \rho_t(a) \geq \varepsilon_t(a) = \min\lrc{\frac{1}{2K}, \frac{1}{2}\sqrt{\frac{\ln K}{tK}}, \xi_t(a)} \geq \min\lrc{\frac{1}{2K}, \frac{1}{2}\sqrt{\frac{\ln K}{tK}}, \frac{\beta \ln t}{t \DLCB_t(a)^2}}$, where the last inequality is by the choice of $\xi_t(a)$. Note that $\DLCB_t(a)$ is a random variable. We derive a high-probability lower bound on the exploration probabilities. 
\begin{definition}
We define the following events:
\begin{align*}
\event(a, t) &= \lrc{\forall s \in \lrc{1,\dots,t}: \varepsilon_s(a) \geq \frac{\beta \ln t}{t \Delta(a)^2}}\\
\event(a^*, a, t) &= \lrc{\forall s \in \lrc{1,\dots,t}: \varepsilon_s(a^*) \geq \frac{\beta \ln t}{t \Delta(a)^2}}.
\end{align*}
\end{definition}

By using Corollary~\ref{cor:DLCB} we get control over the probability of $\event(a,t)$ and $\event(a^*,a,t)$.
\begin{lemma}
For $t \geq \tmin(a)$ (where $\tmin(a)$ is defined in Proposition~\ref{thm:DLCB}) and $\alpha \geq 3$
\begin{align*}
\P{\overline{\event(a,t)}} &\leq \lr{\frac{\ln t}{t \Delta(a)^2}}^{\alpha-2},\\
\P{\overline{\event(a^*,a,t)}} &\leq \lr{\frac{\ln t}{t \Delta(a)^2}}^{\alpha-2}.
\end{align*}
\end{lemma}

\begin{proof}
We start with the proof of the first inequality. Note that by definition $\DLCB_s(a) \leq 1$ and thus $\frac{\beta \ln s}{s \DLCB_s(a)} \geq \frac{\beta \ln s}{s}$. For $s \leq t \Delta(a)^2 / \ln t$ we have $\frac{\beta \ln s}{s} \geq \frac{\beta (\ln s) \ln t}{t \Delta(a)^2} \geq \frac{\beta \ln t}{t \Delta(a)^2}$ and thus $\varepsilon_s(a) \geq \frac{\beta \ln t}{t \Delta(a)^2}$. Therefore, we have:
\begin{align*}
\P{\overline{\event(a,t)}} &=
\P{\exists s \in \lrs{\frac{t \Delta(a)^2}{\ln t}, t}: \varepsilon_s(a) \leq \frac{\beta \ln t}{t \Delta(a)^2}}\\
&= \P{\exists s \in \lrs{\frac{t \Delta(a)^2}{\ln t}, t}: \DLCB_t(a) \geq \Delta(a)} \leq \sum_{s=\frac{t \Delta(a)^2}{\ln t}}^t \frac{1}{s^{\alpha-1}} \leq \frac{1}{2} \lr{\frac{\ln t}{t \Delta(a)^2}}^{\alpha-2}.
\end{align*}
(The summation is bounded by using Lemma~\ref{lem:sum} in the appendix.) The bound for $\P{\overline{\event(a^*,a,t)}}$ follows the same lines using the fact that $\Delta(a^*) = 0 \leq \Delta(a)$.
\end{proof}

\paragraph{Step 3: A lower bound for $N_t(a)$.} We use the following concentration inequality.

\begin{theorem}
\label{thm:SB}
Let $X_1,\dots,X_n$ be Bernoulli random variables adapted to filtration $\F_1,\dots,\F_n$ (in particular, $X_i$ may depend on $X_1,\dots,X_{i-1}$). Let $\event_\gamma$ be the event $\event_\gamma = \lrc{\forall i: \E[X_i \middle | \F_{i-1}] \geq \gamma}$. Then
\[
\P{\lr{\sum_{i=1}^n X_i \leq \frac{1}{2} n\gamma} \land \event_\gamma} \leq e^{-n \gamma / 8}.
\]
\end{theorem}
The result is based on quite standard techniques and its proof is provided in the appendix.

By Theorem~\ref{thm:SB}, for $t \geq t_{\texttt{min}}(a)$ we have:
\begin{align*}
\P{N_t(a) \leq \frac{\beta \ln t}{2\Delta(a)^2}} &\leq \P{\event(a,t) \land \lr{N_t(a) \leq \frac{\beta \ln t}{2\Delta(a)^2}}} + \P{\overline{\event(a,t)}}\\
&\leq e^{-\frac{\beta \ln t}{8\Delta(a)^2}} + \frac{1}{2} \lr{\frac{\ln t}{t \Delta(a)^2}}^{\alpha-2}\\
&\leq \lr{\frac{1}{t}}^{\frac{\beta}{8\Delta(a)^2}} + \frac{1}{2} \lr{\frac{\ln t}{t \Delta(a)^2}}^{\alpha-2}\\
&\leq \lr{\frac{1}{t}}^{\frac{\beta}{8}} + \frac{1}{2} \lr{\frac{\ln t}{t \Delta(a)^2}}^{\alpha-2}.
\end{align*}

In the same way we have $\P{N_t(a^*) \leq \frac{\beta \ln t}{2\Delta(a)^2}} \leq \lr{\frac{1}{t}}^{\frac{\beta}{8}} + \frac{1}{2} \lr{\frac{\ln t}{t \Delta(a)^2}}^{\alpha-2}$.

\paragraph{Step 4: A lower bound for $\DLCB_t(a)$.} 
By Lemma~\ref{cl:1} upper and lower confidence bounds satisfy $\P{\lr{\UCB_t(a^*) \leq \mu(a^*)} \lor \lr{\LCB_t(a) \geq \mu(a)}} \leq \frac{2}{Kt^{\alpha-1}}$. Assuming that $\UCB_t(a^*) \geq \mu(a^*)$ and $\LCB_t(a) \leq \mu(a)$, we have:
\begin{align*}
\DLCB_t(a) &\geq \LCB_t(a) - \min_{a'} \UCB_t(a)\\
&\geq \LCB_t(a) - \UCB_t(a^*)\\
&= \frac{\hat L_{t-1}(a)}{N_{t-1}(a)} - \sqrt{\frac{\alpha \ln \lr{t K^{1/\alpha}}}{2 N_t(a)}} - \frac{\hat L_{t-1}(a^*)}{N_{t-1}(a^*)} - \sqrt{\frac{\alpha \ln \lr{t K^{1/\alpha}}}{2 N_t(a^*)}}\\
&= \frac{\hat L_{t-1}(a)}{N_{t-1}(a)} + \sqrt{\frac{\alpha \ln \lr{t K^{1/\alpha}}}{2 N_t(a)}} - 2\sqrt{\frac{\alpha \ln \lr{t K^{1/\alpha}}}{2 N_t(a)}}\\
&\qquad - \lr{\frac{\hat L_{t-1}(a^*)}{N_{t-1}(a^*)} - \sqrt{\frac{\alpha \ln \lr{t K^{1/\alpha}}}{2 N_t(a^*)}}} - 2 \sqrt{\frac{\alpha \ln \lr{t K^{1/\alpha}}}{2 N_t(a^*)}}\\
&= \UCB_t(a) - \LCB_t(a^*) - 2\sqrt{\frac{\alpha \ln \lr{t K^{1/\alpha}}}{2 N_t(a)}} - 2 \sqrt{\frac{\alpha \ln \lr{t K^{1/\alpha}}}{2 N_t(a^*)}}\\
&\geq \Delta(a) - 2 \sqrt{\frac{\alpha \ln \lr{t K^{1/\alpha}}}{2 N_t(a)}} - 2 \sqrt{\frac{\alpha \ln \lr{t K^{1/\alpha}}}{2 N_t(a^*)}}.
\end{align*}
%Thus, $\P{\DLCB_t(a) \leq \Delta(a) - 2 \sqrt{\frac{\alpha \ln \lr{t K^{1/\alpha}}}{2 N_t(a)}} - 2 \sqrt{\frac{\alpha \ln \lr{t K^{1/\alpha}}}{2 N_t(a^*)}}} \leq \frac{2}{Kt^{\alpha-1}}$.

By Step~3, for $t \geq \tmin(a)$ we have $t \geq K$ and $\P{\lr{N_t(a) \leq \frac{\beta \ln t}{2 \Delta(a)^2}} \lor \lr{N_t(a^*) \leq \frac{\beta \ln t}{2 \Delta(a)^2}}} \leq \lr{\frac{\ln t}{t\Delta(a)^2}}^{\alpha-2} + 2 \lr{\frac{1}{t}}^{\frac{\beta}{8}}$. Assuming that $N_t(a) > \frac{\beta \ln t}{2 \Delta(a)^2}$ and $N_t(a^*) > \frac{\beta \ln t}{2 \Delta(a)^2}$ we have:
\begin{align*}
\DLCB_t(a) &\geq \Delta(a) - 2 \sqrt{\frac{\alpha \ln \lr{t K^{1/\alpha}}}{2 N_t(a)}} - 2 \sqrt{\frac{\alpha \ln \lr{t K^{1/\alpha}}}{2 N_t(a^*)}}\\
&\geq \Delta(a) - 4 \sqrt{\frac{ 2 \Delta(a)^2 \alpha \ln \lr{t K^{1/\alpha}}}{2 \beta \ln t}}\\
&\geq \Delta(a) - 4 \sqrt{\frac{\Delta(a)^2 (\alpha + 1) \ln t}{\beta \ln t}}\\
&= \Delta(a)\lr{1 - 4 \sqrt{\frac{\alpha + 1}{\beta}}}.
\end{align*}
Taking everything together we obtain that for $t \geq \tmin(a)$ and $\beta \geq 64(\alpha+1)$ we have
\[
\P{\DLCB_t(a) \leq \frac{1}{2} \Delta(a)} \leq \lr{\frac{\ln t}{t \Delta(a)^2}}^{\alpha-2} + \frac{2}{Kt^{\alpha-1}} + 2\lr{\frac{1}{t}}^{\frac{\beta}{8}}.
\]

\section{Proof of Theorem~\ref{thm:EXP3++}}
\label{sec:thmEXP3++}

In order to obtain a regret bound, for each suboptimal arm $a$ we have to bound $\sum_{s=1}^t \E[\rho_s(a)]$ and $\sum_{s=1}^t \E[\varepsilon_s(a)]$. For the former we have:
\[
\rho_t(a) = \frac{e^{-\eta_t \tilde L_{t-1}(a)}}{\sum_{a'} e^{-\eta_t \tilde L_{t-1}(a')}} 
= \frac{e^{-\eta_t \lr{\tilde L_{t-1}(a) - \tilde L_{t-1}(a^*)}}}{\sum_{a'} e^{-\eta_t\lr{\tilde L_{t-1}(a')- \tilde L_{t-1}(a^*)}}}
\leq e^{-\eta_t \lr{\tilde L_{t-1}(a) - \tilde L_{t-1}(a^*)}}
= e^{-\eta_t \tilde \Delta_t(a)},
\]
%\begin{align*}
%\rho_t(a) &= \frac{e^{-\eta_t \tilde L_{t-1}(a)}}{\sum_{a'} e^{-\eta_t \tilde L_{t-1}(a')}}\\
%&= \frac{e^{-\eta_t \lr{\tilde L_{t-1}(a) - \tilde L_{t-1}(a^*)}}}{\sum_{a'} e^{-\eta_t \tilde L_{t-1}(a')- \tilde L_{t-1}(a^*)}}\\
%&\leq e^{-\eta_t \lr{\tilde L_{t-1}(a) - \tilde L_{t-1}(a^*)}}\\
%&= e^{-\eta_t \tilde \Delta_t(a)},
%\end{align*}
where $\tilde \Delta_t(a) = \tilde L_{t-1}(a) - \tilde L_{t-1}(a^*)$ is the gap between cumulative importance-weighted estimates of the losses. Unfortunately, the bound on unweighted gap estimates $\DLCB_t(a)$ provided by Proposition~\ref{thm:DLCB} does not directly lead to a bound on the weighted gap estimates $\tilde \Delta_t(a)$ and, therefore, does not provide a bound on $\rho_t(a)$. We use the following form on Bernstein's inequality for martingales to achieve this goal. Theorem~\ref{thm:Bernstein} is a minor variation of a classical Bernstein's inequality for martingales \citep{Fre75}, where we relax the assumption on boundedness of the martingale difference sequence. The theorem follows by a simple adaptation of the proof by \citet[Theorem 3.15]{McD98}, which is sketched in the appendix.

\begin{theorem}[Bernstein's inequality for martingales] Let $X_1,\dots,X_n$ be a martingale difference sequence with respect to filtration $\F_1,\dots,\F_n$, where each $X_j$ is bounded from above, and let $S_i = \sum_{j=1}^i X_j$ be the associated martingale. Let $\nu_n = \sum_{j=1}^n \E[\lr{X_j}^2 \middle | \F_{j-1}]$ and $\displaystyle c_n = \max_{1\leq j\leq n} \lrc{X_j}$. Then for any $\delta > 0$:
\[
\P{\lr{S_n \geq \sqrt{2 \nu \ln \frac{1}{\delta}} + \frac{c \ln \frac{1}{\delta}}{3}} \wedge \lr{\nu_n \leq \nu} \wedge \lr{c_n \leq c}} \leq \delta.
\]
\label{thm:Bernstein}
\end{theorem}

We apply this theorem to martingale difference sequence $X_s = \Delta(a) - \lr{\tilde \ell_s^a - \tilde \ell_s^{a^*}}$ with respect to filtration $\F_1,\F_2,\dots$ in order to bound the martingale $t\Delta(a) - \tilde \Delta_t(a) = \sum_{s=1}^t X_s$. We start by bounding the magnitude of $X_s$-es and the sum of their conditional variances and then use Bernstein's inequality to bound $\tilde \Delta_t(a)$. The bound on $\tilde \Delta_t(a)$ is then used to bound $\rho_t(a)$. At the end we treat the second term of the regret bound, $\sum_{s=1}^t \E[\varepsilon_s(a)]$.

\paragraph{Control of the magnitude of $\displaystyle \max_{1\leq s\leq t} \lrc{X_s}$.} We start by bounding the magnitude of the martingale difference sequence $X_s = \Delta(a) - \lr{\tilde \ell_s^a - \tilde \ell_s^{a^*}}$. We have:
\begin{align*}
\Delta(a) - \lr{\tilde \ell_s^a - \tilde \ell_s^{a^*}} &\leq 1 + \tilde \ell_s^{a^*}\\
&\leq 1 + \frac{1}{\varepsilon_s(a^*)}\\
&= 1 + \max \lrc{2K, 2 \sqrt{\frac{sK}{\ln K}}, \frac{s \DLCB_s(a^*)^2}{\beta \ln s}}\\
&\leq 1.25 \max \lrc{2K, 2 \sqrt{\frac{sK}{\ln K}}, \frac{s \DLCB_s(a^*)^2}{\beta \ln s}}.
\end{align*}
Note that $\DLCB_s(a^*) \leq 1$ and thus for $t \geq \tmin(a)$ and $s \leq t \Delta(a)^2 / \ln t$ we have $\frac{1}{\varepsilon_s(a^*)} \leq \frac{t\Delta(a)^2}{\beta \ln t}$ (we have $a^*$ on the left-hand side and $a$ on the right-hand side). Furthermore, by Proposition~\ref{thm:DLCB} we have that $\P{\DLCB_s(a^*) \geq \Delta(a^*)} \leq \frac{1}{s^{\alpha - 1}}$, where $\Delta(a^*) = 0$. Thus,
\[
\P{\exists s \in \lrs{\frac{t \Delta(a)^2}{\ln t}, t}: \DLCB_s(a^*) \geq 0}
\leq \sum_{s = t \Delta(a)^2 / \ln t}^t \frac{1}{s^{\alpha -1}}
\leq \frac{1}{2}\lr{\frac{\ln t}{t\Delta(a)^2}}^{\alpha - 2}.
\]
Let $c_t = \displaystyle\max_{1\leq s\leq t}\lrc{X_s}$. We obtain that for $t \geq \tmin(a)$
\[
\P{c_t \geq \frac{1.25 t\Delta(a)^2}{\beta\ln t}} \leq \frac{1}{2}\lr{\frac{\ln t}{t\Delta(a)^2}}.
\]

\paragraph{Control of the sum of conditional variances $\sum_{s=1}^t \E[\lr{X_s}^2\middle|\F_{s-1}]$.} We start by looking at individual terms in the sum. We have:
\begin{align*}
\E[\lr{\Delta(a) - \lr{\tilde \ell_s^a - \tilde \ell_s^{a^*}}}^2 \middle |\F_{s-1}]
&\leq \E[\lr{\tilde \ell_s^a - \tilde \ell_s^{a^*}}^2\middle | \F_{s-1}]\\
&= \E[\lr{\tilde \ell_s^a}^2\middle |\F_{s-1}] + \E[\lr{\tilde \ell_s^{a^*}}^2\middle |\F_{s-1}],
\end{align*}
where the equality is due to the fact that by the way importance-weighted samples are defined we have $\tilde \ell_s^a \tilde \ell_s^{a^*} = 0$ and thus $\lr{\tilde \ell_s^a - \tilde \ell_s^{a^*}}^2 = \lr{\tilde \ell_s^a}^2 + \lr{\tilde \ell_s^{a^*}}^2$. 
Further,
\[
\E[\lr{\tilde \ell_s^a}^2\middle |\F_{s-1}] = \tilde \rho_s(a) \lr{\frac{\ell_s^a}{\tilde \rho_s(a)}}^2 \leq \frac{1}{\tilde \rho_s(a)} \leq \frac{1}{\varepsilon_s(a)}
= \max \lrc{2K, 2 \sqrt{\frac{sK}{\ln K}}, \frac{s \DLCB_s(a)^2}{\beta \ln s}}.
\]
Note that $\DLCB_s(a) \leq 1$ and thus for $t \geq \tmin(a)$ and $s \leq t \Delta(a)^2 / \ln t$ we have $\frac{1}{\varepsilon_s(a)} \leq \frac{t\Delta(a)^2}{\beta \ln t}$. Furthermore,
\[
\P{\exists s \in \lrs{\frac{t \Delta(a)^2}{\ln t}, t} : \DLCB_s(a) \geq \Delta(a)}
\leq \sum_{s = \frac{t \Delta(a)^2}{\ln t}}^t \frac{1}{s^{\alpha-1}}
\leq \frac{1}{2}\lr{\frac{\ln t}{t \Delta(a)^2}}^{\alpha - 2}.
\]
%\begin{align*}
%&\P{\exists s \in \lrs{\frac{t \Delta(a)^2}{\ln t}, t} : \DLCB_s(a) \geq \Delta(a)}\\
%&\qquad\leq \sum_{s = \frac{t \Delta(a)^2}{\ln t}}^t \frac{2}{s^{\alpha-1}}\\
%&\qquad\leq \lr{\frac{\ln t}{t \Delta(a)^2}}^{\alpha - 2}.
%\end{align*}

We define $\nu_t = \sum_{s=1}^t \E[\lr{\Delta(a) - \lr{\tilde \ell_s^a - \tilde \ell_s^{a^*}}}^2 \middle | \F_{s-1}]$ and we have that for $t \geq \tmin(a)$
\begin{align*}
&\P{\nu_t \geq \frac{2t^2 \Delta(a)^2}{\beta \ln t}}\\ 
&\qquad\leq \P{\exists s \in \lrs{\frac{t \Delta(a)^2}{\ln t}, t} : \DLCB_s(a) \geq \Delta(a)}+ \P{\exists s \in \lrs{\frac{t \Delta(a)^2}{\ln t}, t} : \DLCB_s(a^*) \geq 0}\\
&\qquad\leq \lr{\frac{\ln t}{t \Delta(a)^2}}^{\alpha - 2}.
\end{align*}
Note that the random event involving $\DLCB_s(a^*)$ is the same as the one we have considered in Step~1. Thus, in total $\P{c_t \geq \frac{1.25t\Delta(a)^2}{\beta \ln t}} + \P{\nu_t \geq \frac{2t^2\Delta(a)^2}{\beta\ln t}} \leq \lr{\frac{\ln t}{t \Delta(a)^2}}^{\alpha - 2}$.

\paragraph{Control of $\tilde \Delta_t(a)$.} We have that
\begin{align*}
&\P{\tilde \Delta_t(a) \leq \frac{1}{2} t \Delta(a)} \\
&\qquad\qquad= \P{t \Delta(a) - \tilde \Delta_t(a) \geq \frac{1}{2} t \Delta(a)} \\
&\qquad\qquad\leq \P{c_t \geq \frac{1.25 t \Delta(a)^2}{\beta \ln t}} + \P{\nu_t \geq \frac{2t^2 \Delta(a)^2}{\beta \ln t}}\\
&\qquad\qquad\quad + \P{\lr{t \Delta(a) - \tilde \Delta_t(a) \geq \frac{1}{2} t \Delta(a)} \wedge \lr{\nu_t \leq \frac{2t^2 \Delta(a)^2}{\beta \ln t}} \wedge \lr{c_t \leq \frac{1.25 t \Delta(a)^2}{\beta \ln t}}}.
\end{align*}
Taking $\nu = \frac{2t^2 \Delta(a)^2}{\beta \ln t}$, $c = \frac{1.25 t \Delta(a)^2}{\beta \ln t}$, and $\delta = \frac{1}{t}$, for $\beta \geq 256$ we have
\[
\sqrt{2 \nu \ln \frac{1}{\delta}} + \frac{c \ln \frac{1}{\delta}}{3} = \sqrt{\frac{4 t^2 \Delta(a)^2 \ln t}{\beta \ln t}} + \frac{1.25 t \Delta(a)^2 \ln t}{3\beta \ln t}
\leq t \Delta(a) \lr{\frac{2}{\sqrt{\beta}} + \frac{1.25}{3 \beta}} \leq \frac{1}{2} t \Delta(a)
\]
and by Bernstein's inequality the last term is bounded by $\frac{1}{t}$.
%\begin{align*}
%\sqrt{2 \nu \ln \frac{1}{\delta}} + \frac{c \ln \frac{1}{\delta}}{3} &= \sqrt{\frac{4 t^2 \Delta(a)^2 \ln t}{\beta \ln t}} + \frac{1.25 t \Delta(a)^2 \ln t}{3\beta \ln t}\\
%&\leq t \Delta(a) \lr{\frac{2}{\sqrt{\beta}} + \frac{1.25}{3 \beta}} \leq \frac{1}{2} t \Delta(a).
%\end{align*}
Overall, for $t \geq \tmin(a)$:
\[
\P{\tilde \Delta_t(a) \leq \frac{1}{2} t \Delta(a)} \leq \lr{\frac{\ln t}{t \Delta(a)^2}}^{\alpha - 2} + \frac{1}{t}.
\]

\paragraph{Control of $\sum_{s=1}^t \E[\rho_s(a)]$.} From here we have for $\eta_t \geq \frac{1}{2}\sqrt{\frac{\ln K}{tK}}$ and $\alpha \geq 3$:
\begin{align*}
\sum_{s=1}^t \E[\rho_s(a)] &\leq \sum_{s=1}^t \E[e^{-\eta_s  \tilde \Delta_s(a)}]\\
&\leq \tmin(a) + \sum_{s=\tmin(a)}^t \lr{e^{-\frac{1}{2} \eta_s s \Delta(a)} + \lr{\frac{\ln s}{s \Delta(a)^2}}^{\alpha - 2} + \frac{1}{s}}\\
&\leq \tmin(a) + \frac{\lr{\lr{\ln t}^2 + \ln t}}{\Delta(a)^2} + \ln t + 1 + \sum_{s=\tmin(a)}^t e^{-\frac{1}{4}\sqrt{\frac{s \ln K}{K}} \Delta(a)}\\
&\leq \frac{\lr{\lr{\ln t}^2 + \ln t}}{\Delta(a)^2} + \ln t + 1 + \frac{16 K}{\Delta(a)^2 \ln K} + 1 + \tmin(a).
\end{align*}

\paragraph{Control of $\sum_{s=1}^t \E[\varepsilon_s(a)]$.} By Proposition~\ref{thm:DLCB}, for $t \geq \tmin(a)$ we have that $\P{\DLCB_t(a) \leq \frac{1}{2}\Delta(a)} \leq \lr{\frac{\ln t}{t\Delta(a)^2}}^{\alpha-2} + \frac{2}{Kt^{\alpha-1}}+2\lr{\frac{1}{t}}^{\frac{\beta}{8}}$. Thus, for $\alpha = 3$ and $\beta = 256$ we have
\begin{align*}
\sum_{s=1}^t \E[\varepsilon_s(a)] &= \sum_{s=1}^t \E[\min\lrc{\frac{1}{2K}, \frac{1}{2} \sqrt{\frac{\ln s}{sK}}, \frac{\beta \ln s}{s\DLCB_s(a)^2}}]\\
&\leq \sum_{s=1}^t \E[\frac{\beta \ln s}{s\DLCB_s(a)^2}]\\
&\leq \tmin(a) + \frac{4\beta\lr{\lr{\ln t}^2 + \ln t}}{\Delta(a)^2} + \sum_{s=\tmin(a)}^t \lr{\lr{\frac{\ln s}{s\Delta(a)^2}}^{\alpha-2} + \frac{2}{Ks^{\alpha-1}}+\lr{\frac{1}{s}}^{\frac{\beta}{8}}}\\
&\leq \tmin(a) + \frac{4\beta\lr{\lr{\ln t}^2 + \ln t}}{\Delta(a)^2} + \frac{\lr{\ln t}^2 + \ln t}{\Delta(a)^2} + \frac{2}{K}\lr{\ln t+1} + \frac{2\pi^2}{3}.
\end{align*}
By combining the bounds on $\sum_{s=1}^t \E[\rho_s(a)]$ and $\sum_{s=1}^t \E[\varepsilon_s(a)]$ we obtain that $\E[N_t(a)] = O\lr{\sum_{a:\Delta(a)>0} \frac{\lr{\ln t}^2}{\Delta(a)^2}} + \tilde O\lr{\frac{K}{\Delta(a)^4}}$, which leads to the statement of the theorem.

\section{Discussion}
\label{sec:discussion}

We have proposed a new algorithm for gap estimation in stochastic environments that can be combined with other randomized algorithms in a modular fashion. The algorithm provides a gap estimate $\DLCB_t(a)$ that satisfies $\frac{1}{2}\Delta(a) \leq \DLCB_t(a) \leq \Delta(a)$ with high probability. %We have shown that the contribution of the exploration strategy of the algorithm to the regret is of order $\frac{\ln t}{\Delta(a)}$ per each suboptimal action $a$. 
We have shown that the algorithm can be combined with the \EXPPP algorithm, leading to $O\lr{\sqrt{Kt\ln K}}$ regret in the adversarial regime and $O\lr{\sum_{a:\Delta(a)>0} \frac{\lr{\ln t}^2}{\Delta(a)}}$ regret in the stochastic regime, where the latter is an improvement by a multiplicative factor of $\ln t$ over \citet{SS14}.

Our work leads to a number of interesting directions for future research. First, there is a question whether the dependence of the regret guarantee on time horizon in the stochastic regime can be reduced down to $\ln t$. We note that \citet{AC16} have a lower bound on achievable regret guarantees in the stochastic regime when simultaneously certain expected regret guarantees against an adaptive adversary or high-probability regret guarantees against an oblivious adversary are required. However, it is still unknown whether $\ln t$ regret in the stochastic regime can be achieved simultaneously with $\sqrt t$ expected regret against an oblivious adversary. While it does not seem possible to achieve it with the \EXPPP algorithm, some modifications of the playing rule, such as the one used in BOA \citep{Win17}, could potentially do better.

A second question is whether improved regret guarantees can be achieved in the moderately contaminated stochastic regime and adversarial regime with a gap. We believe that it might not be possible with gap estimation strategies based on unweighted rewards and that in order to achieve that we should improve gap estimation based on importance-weighted rewards. The analysis technique suggested in our paper could potentially be useful for that.

There are also a number of more technical questions. For example, can we achieve high-probability regret guarantees by turning to modifications of the \EXP algorithm, such as EXP3-IX \citep{Neu15}? Or could we replace $\frac{1}{\Delta(a)}$ factors with more refined measures of complexity, such as those in kl-UCB-type algorithms \citep{CGM+13}?

% Acknowledgments---Will not appear in anonymized version
\acks{We would like to thank Tor Lattimore and anonymous COLT reviewers for valuable suggestions for improvement of the manuscript. G\'abor Lugosi was supported by the Spanish Ministry of Economy and Competitiveness, Grant MTM2015-67304-P and FEDER.}

\bibliography{bibliography}

\appendix

%\section{Difference of Harmonic Numbers}
%
%Let $H_n = \sum_{i=1}^n \frac{1}{i}$ be a sum of the first $n$ elements of a harmonic sequence. We apply the following result by \citet{CQ03} to bound the difference $H_n - H_m$.
%
%\begin{theorem}[\cite{CQ03}]
%For any natural number $n\in\N$, we have
%\[
%\frac{1}{2n + \frac{1}{1-\gamma} - 2} \leq H_n - \ln n - \gamma < \frac{1}{2n+\frac{1}{3}},
%\]
%where $\gamma = 0.577\dots$ denotes Euler's constant. The constants $\frac{1}{1-\gamma} - 2$ and $\frac{1}{3}$ are the best possible.
%\end{theorem}
%
%\begin{corollary}
%For $n > m$:
%\[
%H_n - H_m \geq \ln \frac{n}{m} - \frac{1}{2m + \frac{1}{3}} \geq \ln \frac{n}{m} - 3.
%\]
%\end{corollary}

\section{Proof of Lemma~\ref{cl:1}}

The proof is based on Hoeffding's inequality \cite{Hoe63}.
\begin{theorem}[Hoeffding's inequality]
Let $X_1,\dots,X_n$ be i.i.d.\ random variables, such that $0\leq X_i\leq 1$ and $\E[X_i] = \mu$ for all $i$. Then
\begin{align*}
\P{\frac{1}{n} \sum_{i=1}^n X_i - \mu \geq \sqrt{\frac{\ln \frac{1}{\delta}}{2n}}} &\leq \delta,\\
\P{\mu - \frac{1}{n} \sum_{i=1}^n X_i \geq \sqrt{\frac{\ln \frac{1}{\delta}}{2n}}} &\leq \delta.
\end{align*}
\end{theorem}

\begin{proof}\textbf{of Lemma~\ref{cl:1}} ~
The proof directly follows the analysis of confidence bounds in \citet{ACBF02}. Note that $N_{t-1}(a)$ is a random variable dependent on $\hat L_{t-1}(a)$ and we cannot apply Hoeffding's inequality directly. Let $X_1,\dots,X_t$ be i.i.d.\ random variables with the same distribution as $\ell_1^a$ and let $\hat M_s = \sum_{r=1}^s X_r$. Then
\begin{align*}
\P{\UCB_t(a) \leq \mu(a)} &= \P{\frac{\hat L_{t-1}(a)}{N_{t-1}(a)} + \sqrt{\frac{\alpha \ln\lr{t K^{1/\alpha}}}{2 N_{t-1}(a)}} \leq \mu(a)}\\
&\leq \P{\exists s \in \lrc{1,\dots,t-1}: \frac{\hat M_s}{s} + \sqrt{\frac{\alpha \ln\lr{t K^{1/\alpha}}}{2 s}} \leq \mu(a)}\\
&\leq \sum_{s=1}^{t-1} \P{\mu(a) - \frac{\hat M_s}{s} \geq \sqrt{\frac{\ln\lr{t^\alpha K}}{2 s}}}\\
&\leq \sum_{s=1}^{t-1} \frac{1}{Kt^\alpha}\\
&\leq \frac{1}{Kt^{\alpha-1}}.
\end{align*}
The proof of the second inequality in the lemma is analogous.
\end{proof}

\section{Partial Sum of Reciprocals of Powers of Natural Numbers}

\begin{lemma}
\label{lem:sum}
For $\alpha \geq 2$ and $m \geq 1$:
\[
\sum_{k=m}^n \frac{1}{k^\alpha} \leq \frac{1}{2 m^{\alpha - 1}}.
\]
\end{lemma}

\begin{proof}
We have $2 k^2 \geq k(k+1)$ and $\sum_{k=m}^n \frac{1}{k(k+1)} \leq \frac{1}{m}$ (which is obtained by writing $\frac{1}{k(k+1)} = \frac{1}{k} - \frac{1}{k+1}$). Thus:
\[
\sum_{k=m}^n \frac{1}{k^\alpha} \leq \frac{1}{2} \sum_{k=m}^n \frac{1}{k(k+1)k^{\alpha - 2}}\leq \frac{1}{2 m^{\alpha - 2}} \sum_{k=m}^n \frac{1}{k(k+1)}  \leq \frac{1}{2 m^{\alpha - 1}}.
\]
\end{proof}

\section{Proof of Theorem~\ref{thm:SB}}

\begin{proof}
We start with a bound on a moment generating function of a single Bernoulli random variable $X$. For any $\lambda > 0$ we have
\[
\E[e^{-\lambda X}] \leq \E[1 - \lambda X + \frac{\lambda^2 X^2}{2}] 
= 1 - \lr{\lambda - \frac{\lambda^2}{2}} \E[X]
\leq e^{-\lr{\lambda - \frac{\lambda^2}{2}} \E[X]}.
\]
%\begin{align*}
%\E[e^{-\lambda X}] &\leq \E[1 - \lambda X + \frac{\lambda^2 X^2}{2}]\\
%& = 1 - \lr{\lambda - \frac{\lambda^2}{2}} \E[X]\\
%& \leq e^{-\lr{\lambda - \frac{\lambda^2}{2}} \E[X]}.
%\end{align*}
And, as a consequence,
\[
\E[e^{\lr{\lambda - \frac{\lambda^2}{2}} \E[X] - \lambda X}] \leq 1.
\]
For $\lambda \in (0,2]$ we have $\lambda - \frac{\lambda^2}{2} \geq 0$. Therefore, if $\E[X] \geq \gamma$ then $\E[e^{\lr{\lambda - \frac{\lambda^2}{2}} \gamma - \lambda X}] \leq \E[e^{\lr{\lambda - \frac{\lambda^2}{2}} \E[X] - \lambda X}] \leq 1$.

Let $Z = \1[\event_\gamma]$ denote the indicator random variable of the event $\event_\gamma$. Then for $\lambda \in (0,2]$ we have
\begin{align*}
\E[Z e^{\lr{\lambda - \frac{\lambda^2}{2}}n\gamma - \lambda \sum_{i=1}^n X_i}] &= \E[Ze^{\lr{\lambda - \frac{\lambda^2}{2}}(n-1)\gamma - \lambda \sum_{i=1}^{n-1} X_i}\E[Ze^{\lr{\lambda - \frac{\lambda^2}{2}} \gamma - \lambda X_n} \middle | \F_{n-1}]]\\
&\leq \E[Ze^{\lr{\lambda - \frac{\lambda^2}{2}}(n-1)\gamma - \lambda \sum_{i=1}^{n-1} X_i}\E[e^{\lr{\lambda - \frac{\lambda^2}{2}} \E[X_n|\F_{n-1}] - \lambda X_n} \middle | \F_{n-1}]]\\
&\leq \E[e^{\lr{\lambda - \frac{\lambda^2}{2}}(n-1)\gamma - \lambda \sum_{i=1}^{n-1} X_i}]\\
&\leq \cdots\\
&\leq 1.
\end{align*}
By combining this result with Markov's inequality we have that for any $\lambda \in (0,2]$
\begin{align*}
\P{\lr{\sum_{i=1}^n X_i \leq n\gamma - \frac{\lambda}{2}n\gamma - \frac{\ln \frac{1}{\delta}}{\lambda}} \land \event_\gamma} &= \P{\lr{\lr{\lambda - \frac{\lambda^2}{2}} n \gamma - \lambda \sum_{i=1}^n X_i \geq \ln \frac{1}{\delta}} \land \event_\gamma} \\
&= \P{Z e^{\lr{\lambda + \frac{\lambda^2}{2}} n \gamma - \lambda\sum_{i=1}^n X_i} \geq \frac{1}{\delta}} \\
&\leq \delta \E[Ze^{\lr{\lambda + \frac{\lambda^2}{2}}n\gamma - \lambda\sum_{i=1}^n X_i}]\\
&\leq \delta.
\end{align*}
By taking $\lambda = \sqrt{\frac{n\gamma}{2 \ln \frac{1}{\delta}}}$ we obtain
\[
\P{\lr{\sum_{i=1}^n X_i \leq n\gamma - \sqrt{2 n\gamma \ln \frac{1}{\delta}}} \land \event_\gamma} \leq \delta.
\]
Finally, taking $\delta = e^{-n\gamma / 8}$ leads to $\lambda = 2$ and completes the proof.
\end{proof}

\section{Proof sketch of Theorem~\ref{thm:Bernstein} (Bernstein's inequality)}

The proof is analogous to the proof of \citet[Theorem 3.15]{McD98}. \citeauthor{McD98} assumes that $X_j$-s are bounded by $c$ and the proof is based on defining an indicator random variable $Z = \1[\nu_n \leq \nu]$ and bounding $\P{\lr{S_n \geq \alpha} \wedge \lr{\nu_n \leq \nu}} = \P{Ze^{\lambda S_n} \geq e^{\lambda \alpha}}$ for $\lambda > 0$. We remove the assumption and define an indicator random variable $Z' = \1[\lr{\nu_n \leq \nu} \wedge \lr{c_n \leq c}]$. Then $\P{\lr{S_n \geq \alpha} \wedge \lr{\nu_n \leq \nu} \wedge \lr{c_n \leq c}} = \P{Z'e^{\lambda S_n} \geq e^{\lambda \alpha}}$ for $\lambda > 0$ and the rest of the proof is identical.

\end{document}